%% file: main.tex
\documentclass{article}

\usepackage{microtype}
\usepackage{graphicx}
\usepackage{subfigure}
\usepackage{booktabs} 
\usepackage{multirow}

\input{math_commands}

\usepackage{amsthm}
\usepackage{amssymb}
\usepackage{enumitem}

\usepackage[algo2e,ruled,vlined,linesnumbered]{algorithm2e}

\newtheorem{theorem}{Theorem}
\newtheorem{definition}{Definition}[section]

\makeatletter
\newcommand{\subalign}[1]{%
  \vcenter{%
    \Let@ \restore@math@cr \default@tag
    \baselineskip\fontdimen10 \scriptfont\tw@
    \advance\baselineskip\fontdimen12 \scriptfont\tw@
    \lineskip\thr@@\fontdimen8 \scriptfont\thr@@
    \lineskiplimit\lineskip
    \ialign{\hfil$\m@th\scriptstyle##$&$\m@th\scriptstyle{}##$\hfil\crcr
      #1\crcr
    }%
  }%
}
\makeatother

\usepackage{hyperref}

\SetKwInput{KwInput}{input}                
\SetKwInput{KwOutput}{return}              

\usepackage[accepted]{icml2021}

\icmltitlerunning{Scalable Pareto Front Approximation for Deep Multi-Objective Learning}

\begin{document}

\twocolumn[
\icmltitle{Scalable Pareto Front Approximation for Deep Multi-Objective Learning}



\icmlsetsymbol{equal}{*}

\begin{icmlauthorlist}
\icmlauthor{Michael Ruchte}{uf}
\icmlauthor{Josif Grabocka}{uf}
\end{icmlauthorlist}

\icmlaffiliation{uf}{Representation Learning Lab (RELEA), Department of Computer Science, University of Freiburg, Germany}

\icmlcorrespondingauthor{Michael Ruchte}{ruchtem@cs.uni-freiburg.de}

\icmlkeywords{Multi-Objective Optimization, Deep Learning}

\vskip 0.3in
]



\printAffiliationsAndNotice{}  

\setcounter{footnote}{1}
\begin{abstract}
Multi-objective optimization (MOO) is a prevalent challenge for Deep Learning, however, there exists no scalable MOO solution for truly deep neural networks. Prior work either demand optimizing a new network for every point on the Pareto front, or induce a large overhead to the number of trainable parameters by using hyper-networks conditioned on modifiable preferences. In this paper, we propose to condition the network directly on these preferences by augmenting them to the feature space. 
Furthermore, we ensure a well-spread Pareto front by penalizing the solutions to maintain a small angle to the preference vector. In a series of experiments, we demonstrate that our Pareto fronts achieve state-of-the-art quality despite being computed significantly faster. Furthermore, we showcase the scalability as 
our method approximates the full Pareto front on the CelebA dataset with an EfficientNet network at a tiny training time overhead of 7\% compared to a simple single-objective optimization. We will make our code publicly available.\footnote{https://github.com/ruchtem/cosmos}
\end{abstract}

\section{Introduction}
\label{introduction}

Multi-objective optimization (MOO) is a pillar task for Machine Learning and manifests itself in numerous real-life problems, such as multi-task learning (MTL) or \textit{fair} machine learning which jointly minimizes both the classification error and at least one fairness criterion. MOO is a classical problem in operations research and is typically solved through evolutionary methods, however, it recently gained interest in the Machine Learning community thanks to the development of new gradient-based MOO algorithms that facilitate faster training times~\cite{fliege2000steepest, desideri2012mutiple}. The goal of MOO is to generate a Pareto front of non-dominating solutions, in a way that a practitioner selects a post hoc solution based on the achieved trade-offs among objectives. The Pareto front denotes a set of solutions which cannot be further optimized without sacrificing performance with respect to at least one objective, its elements are called pareto-optimal (see Section~\ref{sec:preliminary}).

The earlier attempt to tackle MOO for Deep Learning focused on addressing multi-task learning~\citep{sener2018multi} through gradient descent~\cite{desideri2012mutiple} to find a single solution on the Pareto front. Additional follow-up strategies propose to populate a set of Pareto optimal solutions by learning multiple neural networks along preference vectors~\cite{lin2019pareto, mahapatra2020multi}. Preference vectors encode the predefined relative importance for each objective. Two recent ideas proposed conditioning the network's weights to the preference vector through hyper-networks~\cite{lin2021controllable,navon2020learning}. 
A key problem with the aforementioned prior work is that they struggle to scale to deep neural networks. Learning one new network for each Pareto front solution makes it infeasible to populate a large set, while the overhead of the hyper-networks is prohibitive in terms of increasing the number of trainable parameters.


In this paper we propose a novel method that scales MOO to Deep Learning by fulfilling two important desiderata: $i)$ our method does not significantly increase the trainable parameters of the network (contrary to hyper-networks), and $ii)$ generates the full Pareto front of solutions in a single optimization run (contrary to methods that train one network per point in the Pareto front). Concretely, we propose to condition a prediction model to the choice of the preference vectors by augmenting the feature space with the preferences. As a consequence, our method learns to adapt a single network for all the trade-off combinations of the inputted preference vectors, therefore it is able to approximate all solutions of the Pareto front after a single optimization run. We follow the linear scalarization variant of MOO, which although is the fastest variant for gradient-based learning, does not produce a well-spread Pareto front~\cite{mahapatra2020multi,navon2020learning,lin2019pareto}. To remedy this disadvantage, we propose a novel penalty term that forces the achieved solutions in the objectives' space to maintain a small angle to the inputted preference vector.

In a series of experiments, we demonstrate that our method competes strongly with the state-of-the-art MOO methods in terms of quality while being significantly faster. To further showcase the efficiency of our technique, we demonstrate that it is feasible to efficiently approximate the Pareto front by fine-tuning a deep network on a large-scale multi-objective problem within less than 4 GPU hours. To sum up, this paper introduces the following contributions:

\begin{enumerate}
\item A novel method for MOO in Deep Learning that needs a single optimization run to approximate the full set of the Pareto front;
\item Conditioning the model to the preference vectors on the feature space, removing the overhead of hyper-networks on the number of trainable parameters;
\item A novel penalty term to enforce that the Pareto front achieved through linear scalarization is well-spread across the objectives' space;
\item Experimental results showing our method is an efficient MOO approach that approximates a Pareto front for deep models on large multi-objective datasets at negligible training time overhead of 7\%  compared to a single-objective optimization. 
\end{enumerate}

\section{Related Work}
\label{relatedwork}

\textbf{Multi-Task Learning} refers to learning a prediction model for solving multiple tasks jointly and is applicable in various application domains~\cite{zhang2018survey}, for instance in detecting facial landmarks on images~\cite{10.1007/978-3-319-10599-4_7}. Tasks can have competing objectives, and finding the right balance for prioritizing them is challenging. A prior work explores mechanisms for deciding which tasks should be prioritized~\cite{standley2020which}. Furthermore, the adaptive load balancing of competing losses has also attracted interest~\cite{pmlr-v80-chen18a}. On the other hand, a tailored architecture for addressing multi-task involved task-specific feature-level attention \cite{8954221}. Besides, a prior work explores the direction of modeling uncertainty for multi-task learning~\cite{8578879}.

\textbf{Multi-Objective Optimization} formalizes the problem of learning from multiple objectives as discovering a set of Pareto optimal solutions~\cite{Kais99,10.5555/1121732} expressing trade-offs between the objectives (see Section~\ref{sec:preliminary}). Classical approaches follow genetic algorithms that search for populations of Pareto optimal solutions~\cite{10.1007/3-540-45356-3_83}, however, evolutionary algorithms are not scalable for deep learning in an off-the-shelf manner. 

A notable method that explores gradient-based learning for multi-objective optimization is the Multiple-Gradient-Descent-Algorithm (MGDA)~\cite{desideri2012mutiple}. Actually, one of the first papers to treat multi-task learning as multi-objective optimization used the MGDA method for training a single Pareto stationary solution~\cite{sener2018multi}. Follow-up techniques extended the idea towards learning a Pareto front of multiple solutions by aligning solutions according to preference rays~\cite{lin2019pareto,mahapatra2020multi}. Unfortunately, these approaches train one neural network from scratch for each point on the Pareto front. A strategy to speed-up the generation of Pareto stationary points explores new solutions (networks) by taking different directions in the multi-objective space and transfer-learning the weights of past solutions~\cite{pmlr-v119-ma20a}.

Orthogonal to MTL is the application of MOO to fairness objectives which gained interest only recently \cite{valdivia2020fair, padh2020addressing}. The key difference is that in fairness all objectives are based on the same output thus enforcing their trade-offs \cite{menon2018cost}.

The idea of training a single network that is conditioned on a particular task has been first elaborated in the context of multi-task learning~\cite{Dosovitskiy2020You}. Extensions to the case of multi-objective optimization with Pareto fronts have utilized hyper-networks which output a preference-dictated prediction model~\cite{lin2021controllable,navon2020learning}. Hyper-networks are trained to predict the weights of another network conditioned on some input \cite{ha2016hypernetworks}. While our method shares the same style of conditioning the model on the objective preferences, we delineate in two key points: firstly, our solution conditions the feature space instead of the parameter space, and consequently is not limited by the overhead of a hyper-network. Secondly, we propose a novel penalty term that improves the quality of Pareto fronts trained through linear scalarization.

\textbf{Deep Multi-Objective Optimization} relies on applying the aforementioned ideas to the case of deep neural networks. MOO has been a topic of interest for related sub-problems, such as reinforcement learning~\cite{10.5555/3176748.3176753, 10.5555/2627435.2750356, DBLP:conf/nips/YangSN19}, or neural architecture search~\cite{elsken2018efficient}. However, we assess that prior work can not be easily deployed to state-of-the-art deep networks, because they either need to train multiple networks for populating the Pareto front or are limited by their dependence on parameter-heavy hyper-networks. In contrast, our method (Section~\ref{proposedmethod}) produces a Pareto front within a time complexity comparable to a simple single-objective optimization with one network.

\section{Preliminary}
\label{sec:preliminary}

Let us denote a supervised dataset as $D:=\left\{\left(x_n, y_n\right)\right\}_{n=1}^N$, where $x_n \in \mathcal{X}, y_n \in \mathcal{Y}, \forall n \in \left\{1,\dots, N\right\}$. We are asked to train a neural network $f(x,\theta): \mathcal{X} \times \Theta \rightarrow \mathcal{Y}$ parameterized by $\theta \in \Theta$ by jointly minimizing a set of $J$ loss functions $\Ls_j: \mathcal{Y} \times \mathcal{Y} \rightarrow \R_{>0}, \forall j \in \left\{1, \dots, J\right\}$ as formulated in the \textbf{multi-objective} problem of  Equation~\ref{eq:moo}. Each objective focuses on minimizing the expectation of the respective loss function over labeled instances $x, y$ drawn from a data sampling distribution $p_D$. 

\begin{align}
    \label{eq:moo}
    \nonumber
    \theta^{*} := \argmin_{\theta} \mathbb{E}_{\left(x,y\right) \sim p_D} \left\{ \right. &\Ls_1\left(y, f\left(x; \theta \right)\right),  \\
    \nonumber
    &\Ls_2\left(y, f\left(x; \theta \right)\right), \\ 
    \nonumber
    &\dots,  \\ 
     &\Ls_J\left(y, f\left(x; \theta \right)\right) \left.\right\}
\end{align}

In general, no single solution $\theta^{*}$ achieves the optimum of all objectives, however, we can obtain a set of Pareto optimal solutions according to the following definitions. For ease of notation, let us denote an objective as $\mathcal{O}_j(\theta) := \mathbb{E}_{\left(x,y\right) \sim p_D} \Ls_j\left(y, f\left(x; \theta\right)\right), \; \forall j \in \left\{1, \dots, J\right\}$. 

\smallskip

\theoremstyle{definition}
\begin{definition}[Pareto dominance] A solution $\theta \in \Theta$ dominates another solution $\theta' \in \Theta$ (denoted as $\theta \prec \theta'$) when both: \textbf{i)} $\theta$ is not worse than $\theta'$ on any objective, i.e.: $\mathcal{O}_j(\theta) \le \mathcal{O}_j(\theta'), \;  \forall j \in \left\{1, \dots. J\right\}$, and: \textbf{ii)} $\theta$ is better than $\theta'$ on at least one objective, i.e.: $\exists k \in \left\{1, \dots, J\right\} \text{ s.t. } \mathcal{O}_k(\theta)  < \mathcal{O}_k(\theta')$.
\label{def:paretodominance}
\end{definition}

\theoremstyle{definition}
\begin{definition}[Pareto optimality] A solution $\theta \in \Theta$ is Pareto optimal if it is not dominated by any other solution. Therefore, the set of all Pareto optimal solutions is defined as $\mathcal{P} := \left\{\theta \in \Theta \; | \; \nexists  \theta' \in \Theta: \; \theta' \prec \theta \right\}$. Meanwhile, the Pareto front $\mathcal{F}$ is the $J$-dimensional manifold of the objective values of all Pareto optimal solutions $\mathcal{F} := \left\{\mathcal{O}(\theta) \in \R_{>0}^J \; | \; \theta \in \mathcal{P} \right\}$.
\label{def:paretooptimality}
\end{definition}

The problem of multi-objective optimization (MOO) can be treated as single-objective optimization through the \textbf{linear scalarization} problem of Equation~\ref{eq:mools}, given a preference vector $r \in \R_{>0}^J$.

\begin{align}
    \label{eq:mools}
    \theta^{*}_r := \argmin_{\theta} \; \mathbb{E}_{\left(x,y\right) \sim p_D} \;\;  \sum_{j=1}^{J} r_j \; \Ls_j\left(y, f\left(x; \theta\right)\right)
\end{align}

\begin{theorem}
The optimal solution $\theta^{*}_r$ of Equation~\ref{eq:mools} is Pareto-optimal for any given $r \in \R_{>0}^J$.
\end{theorem}

\begin{proof}
We can show that no another solution $\theta'$ dominates $\theta^{*}_r$ using a commonly known proof. If there exists a different solution $\theta'$ that dominates the optimal linear scalarization solution $\theta^{*}$, then by virtue of the Pareto optimality definition $\forall j \in \left\{1, \dots. J\right\}: \mathcal{O}_j(\theta')  \le \mathcal{O}_j(\theta^{*}_r)$ and $\exists k \in \left\{1, \dots. J\right\}: \mathcal{O}_k(\theta')  < \mathcal{O}_k(\theta^{*}_r)$. If these optimality conditions hold and knowing $r \in \R_{>0}^J$, then  $\sum_{j=1}^{J} r_j \mathcal{O}_j(\theta') < \sum_{j=1}^{J} r_j \mathcal{O}_j(\theta^{*}_r)$ must be true, which is not possible because $\theta^{*}_r$ is the optimal solution of Equation~\ref{eq:mools}. Therefore, the above Pareto optimality conditions can not hold and a dominating $\theta'$ does not exist. 
\end{proof}

However, the loss functions $\Ls_j: \mathcal{Y} \times \mathcal{Y} \rightarrow \R_{>0}$, $j \in \left\{1, \dots. J\right\}$ are non-convex with respect to $\theta \in \Theta$ in the case when $f(x,\theta): \mathcal{X} \times \Theta \rightarrow \mathcal{Y}$ is a neural network. Therefore, neural network weights $\theta^{*}_r$ computed by first-order optimization approaches are not the optimal solution of Equation~\ref{eq:mools}. As a result, optimizing neural networks through the linear scalarization approach yields approximative Pareto-optimal solutions. Consequently, we can create an approximation to the Pareto optimal set of solutions by computing the optimal parameters $\theta^{*}_r$ of Equation~\ref{eq:mools} for varying preference vectors $r$ as $\hat{\mathcal{P}} := \left\{ \theta^{*}_r \in \Theta \; | \; r \in \R_{>0}^{J} \right\}$. Unfortunately, such a strategy requires repeating Equation~\ref{eq:mools} multiple times with randomly sampled preference vectors $r$, which is computationally intractable as $|\hat{\mathcal{P}}| \gg 1$.

\section{Proposed Method}
\label{proposedmethod}

We propose to approximate the Pareto front via a one-shot optimization procedure. Instead of solving multiple linear scalarization problems (Equation~\ref{eq:mools}) with different sampled $r$, we condition the predictions of the neural network to the preference vectors as $f(x, r, \theta): \mathcal{X} \times \R_{>0}^J \times \Theta \rightarrow \mathcal{Y}$. Practically, we concatenate the input $x$ with the vector $r$ and train a neural network on this joint feature space. In that manner, we can learn a single network $f$ whose predictions are optimized to achieve the Pareto front solution for the respective inputted preference vectors. The objective for optimizing the conditioned model is shown in Equation~\ref{eq:oneshotmoo}. We restrict the choice of $r \in \left[0,1\right]^J, \sum_{j=1}^{J} r_j=1$ by sampling from a Dirichlet distribution controlled with $\alpha \in \R_{>0}^J$.

\begin{align}
    \label{eq:oneshotmoo}
    \theta^{*} := \argmin_{\theta} \; \mathbb{E}_{\;\subalign{ &\bm{r }\sim \text{Dir}\left(\alpha \right) \\ &\left(x,y\right) \sim p_D}} \;  \sum_{j=1}^{J} r_j \; \Ls_j\left(y, f\left(x, \bm{r}; \theta\right) \right)
\end{align}

Let us denote any objective of the conditioned neural network with parameters $\theta$ given $r$ as $\mathcal{O}_j(\theta, r) := \mathbb{E}_{\left(x,y\right) \sim p_D} \Ls_j\left(y, f\left(x, r; \theta\right)\right), \forall j \in \left\{1, \dots, J\right\}$ and the cumulative objective vector as $\mathcal{O}(\theta, r)  \in \R_{>0}^J$. Once we compute the optimal solution $\theta^{*}$ of Equation~\ref{eq:oneshotmoo}, we can afterwards effortlessly approximate the Pareto front as formulated in Equation~\ref{eq:oneshotparetofront}. 

\begin{align}
\label{eq:oneshotparetofront}
\hat{\mathcal{F}} := \left\{\mathcal{O}(\theta^{*}, r) \in \R_{>0}^J \; | \; r \sim \text{Dir}\left(\alpha \right) \right\}
\end{align}

The advantage of our approximation method is that creating the Pareto front demands a single conditioned neural network with different randomly sampled $r$ vectors, which is more efficient than optimizing one new network from scratch for every sampled $r$. To generate the Pareto front we need just one gradient-based optimization run on the linear scalarization of Equation~\ref{eq:oneshotmoo}.

Unfortunately, although the linear scalarization approach produces Pareto-optimal solutions, still the Pareto Front is very narrow, a phenomenon witnessed by multiple authors~\cite{mahapatra2020multi,navon2020learning,lin2019pareto}. Prior work remedy this behavior by forcing the solutions to collocate with a wide set of predefined rays on the loss space~\cite{lin2019pareto,mahapatra2020multi}.

We propose a novel penalty term for forcing the solution to obey to the preference vector $r$ by minimizing the angle between $r$ and the vector of losses $\vec \Ls\left(x, y, r; \theta\right) := \left[ \left( \Ls_j\left(y, f\left(x, r; \theta\right) \right) \right)_{j=1}^J \right]$. Such a desired effect can be modeled by maximizing the cosine similarity between the preference vector and the vector of losses as $\frac{r^T \vec \Ls\left(x, y, r; \theta\right) }{||r|| \; ||\vec \Ls\left(x, y, r; \theta\right) ||}$.

\begin{align}
    \label{eq:oneshotmoopenalty}
     \nonumber
    \theta^{*} := \argmin_{\theta} \; \mathbb{E}_{\;\subalign{ &r \sim \text{Dir}\left(\alpha\right) \\ &\left(x,y\right) \sim p_D}} \;\;\; & r^T \vec \Ls\left(x, y, r; \theta\right) \\
    - \lambda & \frac{r^T \vec \Ls\left(x, y, r; \theta\right) }{||r|| \; ||\vec \Ls\left(x, y, r; \theta\right) ||} 
\end{align}

The resulting objective function is shown in Equation~\ref{eq:oneshotmoopenalty}, where we both minimize the losses through the term $r^T \vec \Ls\left(x, y, r; \theta\right)$, as well as maintaining a well-spread Pareto front through maximizing the cosine similarity. We can control the magnitude of the cosine similarity regularization via a penalty hyperparameter $\lambda \in \R_{>0}$.

\begin{algorithm2e}[ht!]
\DontPrintSemicolon
\KwInput{Labeled dataset: $\left\{\left(x_n, y_n\right)\right\}_{n=1}^N$, Loss functions: $\left\{ \Ls_j \right\}_{j=1}^J$, Learning rate scheduler: $\eta: \mathbb{N} \rightarrow \R_{>0}$, Number of steps: $K$, Penalty: $\lambda \in \R_{>0}$, Sampling: $\alpha \in \R_{>0}^J$.}
 \medskip
 \text{Initialize} $\theta \in \Theta$ \;
 \medskip
 \For{$i=1,\dots,K$}{
 \medskip
  $r \sim \text{Dir}\left(\alpha \right)$ \; 
  \medskip
  $\left(x,y\right) \sim p_D$ \;
  \medskip
  $\vec \Ls\left(x, y, r; \theta\right) := \left[ \left( \Ls_j\left(y, f\left(x, r; \theta\right) \right) \right)_{j=1}^J \right]$ \;
  \medskip
  $g := \nabla_\theta \left[ r^T \vec \Ls\left(x, y, r; \theta\right) - \lambda \frac{r^T \vec \Ls\left(x, y, r; \theta\right) }{||r|| \; ||\vec \Ls\left(x, y, r; \theta\right) ||} \right]$ \;
  \medskip
  $\theta \leftarrow \theta - \eta(i) \; g$ \;
 }
 \KwOutput{$\theta$}
 \caption{Efficient One-shot MOO}
 \label{alg:moo}
\end{algorithm2e}

The actual gradient-based optimization procedure for the conditioned prediction model is detailed in Algorithm~\ref{alg:moo}, where the procedure iterates through a sequence of mini-batches $(x,y)$ and at every step draws a random preference vector $r$ from a Dirichlet distribution. The conditioned neural network parameters are updated using the gradient of Equation~\ref{eq:oneshotmoopenalty} on the mini-batch objectives given the sampled preference vector.

In this manner, optimizing our method has an algorithmic complexity independent of $J$. 
In particular, it is worth highlighting that we need only a single back-propagation through the neural network $f$ for the sum of the $J$ losses, not $J$ back-propagations for every loss. Therefore, the proposed method is asymptotically as fast as learning a simple single-objective problem. Hence, we offer a \textit{free lunch} multi-objective optimization for deep learning from a runtime perspective.   




\begin{figure*}[ht!]
    \centering
    \includegraphics[width=0.95\textwidth]{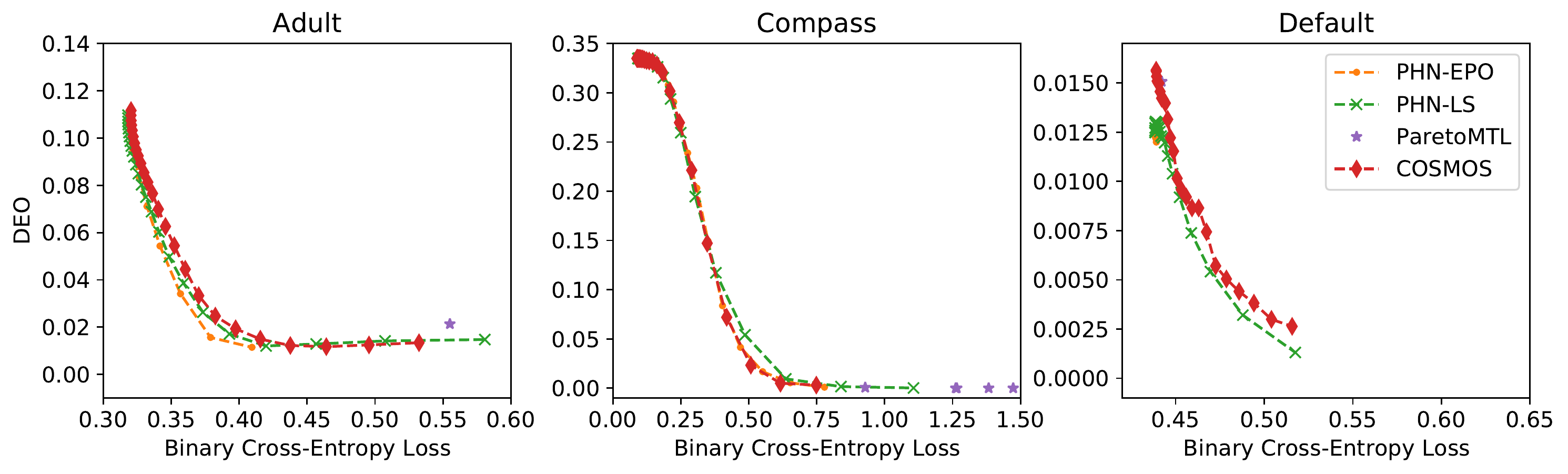}
    \includegraphics[width=0.95\textwidth]{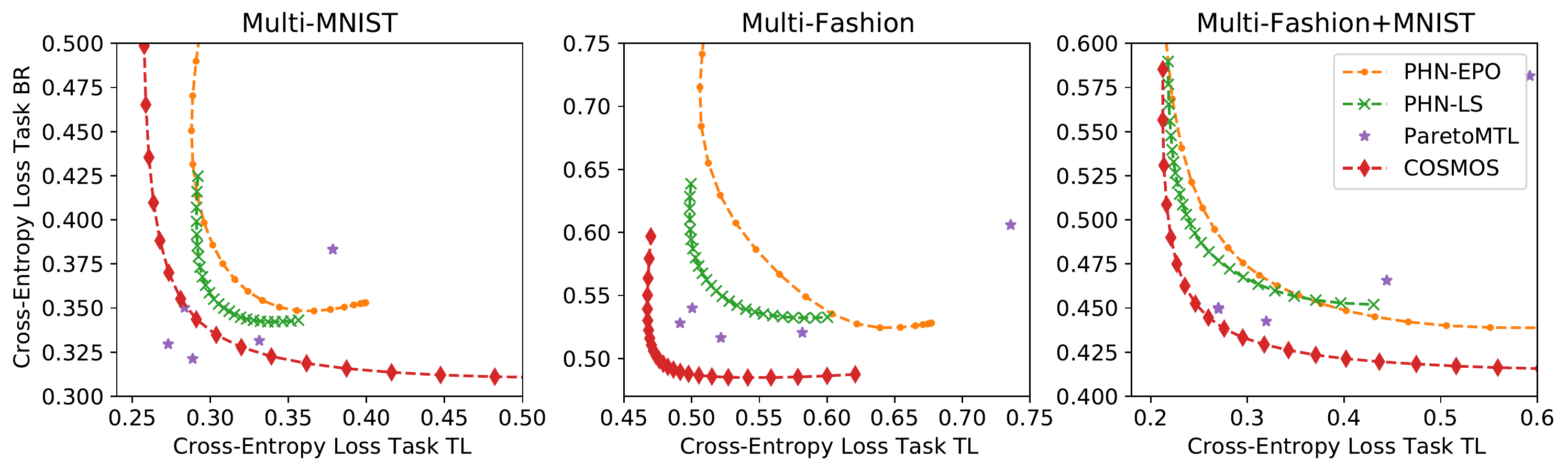}
    \vspace{-0.3cm}
    \caption{Comparison of the Pareto fronts on fairness and image classification datasets against SOTA baselines}
    \label{fig:vssota}
\end{figure*}

\section{Experimental Protocol}
\label{experimentalprotocol}

\subsection{Research Hypotheses and Designed Experiments}

We analyse our method (named COSMOS\footnote{COSMOS:  \textit{Conditioned One-shot Multi-Objective Search}}) along the following hypotheses:

\begin{itemize}
    \item \textbf{Hypothesis 1:} Does our method generate qualitative Pareto fronts compared to the state-of-the-art techniques while being faster in terms of training time? \\ \textbf{Experiment 1:} We compare our method against state of the art baselines: PHN-LS, PHN-EPO \cite{navon2020learning}, ParetoMTL \cite{lin2019pareto}, which represent MOO baselines that are able to generate Pareto fronts with gradient-based methods. We use two sets of publicly available multi-objective datasets: in the realm of fair classification (Adult, Compass, and Default) and image classification (three variants of Multi-MNIST).
    \item \textbf{Hypothesis 2:} Does our proposed MOO approach compete strongly against single-objective baselines? \\  \textbf{Experiment 2:} We compare our Pareto fronts against single-objective baselines on the datasets mentioned in Experiment 1.
    \item \textbf{Hypothesis 3:} Does the proposed cosine similarity penalty term achieve a wider-spread Pareto front? \\ \textbf{Experiment 3:} We perform an ablation study on different values of $\lambda$ and $\alpha$ to understand its effect on the generated Pareto front for the image classification datasets.
    \item \textbf{Hypothesis 4:} Can our approach scale to large datasets and optimize state of the art deep neural networks? \\ \textbf{Experiment 4:} We train EfficientNet-B4 on the larger CelebA dataset and demonstrate the quality of the achieved Pareto front as well as the training time.
\end{itemize}

\subsection{Experimental Setup}

The experiments aim at demonstrating the scalability as well as the quality of the results of our method. Unfortunately, these two requirements are mutually infeasible. If the state-of-the-art methods cannot scale to very deep models, then we cannot assess the comparative quality of the Pareto fronts generated by COSMOS against the prior work on the large-scale MOO setup. As a solution, we follow the strategy below:

\begin{enumerate}[leftmargin=*]
    \item We show the quality of the Pareto fronts generated by COSMOS \textit{on the identical small-scale experimental setup} as in the published papers of the baselines. We clarify that all the three baselines (PHN-LS, PHN-EPO, ParetoMTL) used a toy-scaled LeNet architecture for the published experiments with small-size datasets (Multi-Mnist and Multi-Fashion).  
    \item After we show the quality of the Pareto fronts generated by COSMOS versus the state-of-the-art on the small-scale setup, then we demonstrate the \textit{scalability} of COSMOS on the larger CelebA dataset with a deep EfficientNet model. It is important to highlight that none of the baselines can scale to this experiment. ParetoMTL needs to train one different EfficientNet network for each point on the Pareto front, while PHN-LS and PHN-EPO increase the number of trainable parameters to approximately $100 \times$ that of an EfficientNet architecture.
\end{enumerate}

\begin{table*}[htb!]
\caption{Results compared to the state-of-the-art methods (HV: hyper-volume, Training time in sec)}
\centering
\begin{tabular}{l rr | rr | rr | rr | rr | rr }
\toprule
             \multirow{2}{*}{\bf Method}    & \multicolumn{2}{c}{\bf Adult} & \multicolumn{2}{c}{\bf Compass} & \multicolumn{2}{c}{\bf Default} & \multicolumn{2}{c}{\bf Multi-MNIST} & \multicolumn{2}{c}{\bf Multi-Fashion} & \multicolumn{2}{c}{\bf Fash.+MNIST}  \\ \cmidrule{2-13}
                & \multicolumn{1}{c}{\bf HV} & \multicolumn{1}{c}{\bf Time} & \multicolumn{1}{c}{\bf HV} & \multicolumn{1}{c}{\bf Time} & \multicolumn{1}{c}{\bf HV} & \multicolumn{1}{c}{\bf Time} & \multicolumn{1}{c}{\bf HV} & \multicolumn{1}{c}{\bf Time} & \multicolumn{1}{c}{\bf HV} & \multicolumn{1}{c}{\bf Time} & \multicolumn{1}{c}{\bf HV} & \multicolumn{1}{c}{\bf Time} \\ \midrule 
Single Task    & -        & 48     & - & 16     &  -       & 34 & 2.85       & 391 & 2.23 & 421 & 2.77 & 398 \\ 
PHN-EPO    & \textbf{3.34}      & 59     &    \textbf{3.71}       &  17 & 3.11 & 27 & 2.83 & 1554 & 2.20 & 1852 & 2.78 & 1715  \\ 
PHN-LS    & \textbf{3.34}    & 23 & \textbf{3.71} & 7 & \textbf{3.12} & 18 & 2.83 & 636 & 2.20 & 700 & 2.76 & 723 \\ 
ParetoMTL    & 2.90    & 539 & 2.15 & 83 & 3.10 & 334 & 2.90 & 4087 & 2.24 & 4210 & 2.72 & 4142 \\ \midrule
\textbf{COSMOS}    & \textbf{3.34}        & 31 & \textbf{3.72} & 17 & \textbf{3.12} & 35 & \textbf{2.94} & 501 & \textbf{2.32} & 379 & \textbf{2.83} & 498 \\ 
\bottomrule
\end{tabular}
\label{tab:results_fair_mnist}
\end{table*}

\begin{figure*}[htb!]
    \centering
    \includegraphics[width=0.95\textwidth]{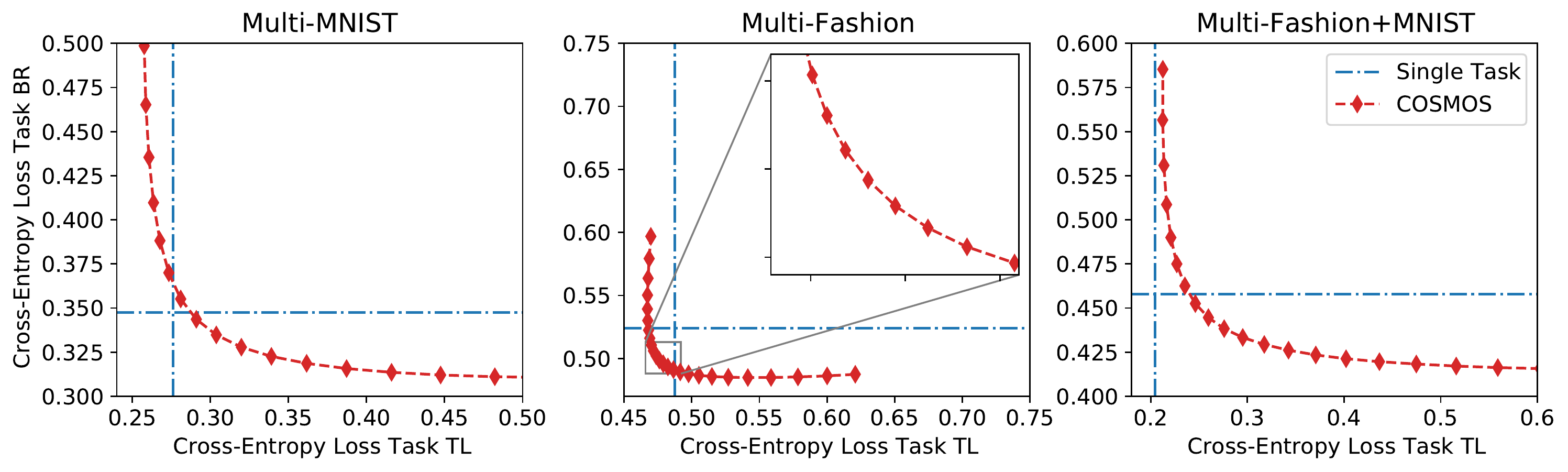}
    \vspace{-0.3cm}
    \caption{Comparison of COSMOS against single-objective baselines}
    \label{fig:vssingle}
\end{figure*}

\subsection{Reproducibility}
In terms of reproducibility, we set the hyperparameters for the baselines following the values defined in their papers and official implementations (details in the appendices). To make the baselines even more competitive, we perform early stopping using hyper-volume~\cite{zitzler2007hypervolume} computed on the validation set. Unless mentioned otherwise we set the batch size to 256 and use Adam~ \cite{DBLP:journals/corr/KingmaB14} with a learning rate of $10^{-3}$. We report the mean scores and Pareto fronts over 5 independent runs for each method and dataset, apart from CelebA. We use (2, 2) as the reference point for hyper-volume and 25 equally distributed test rays.

\textbf{Fairness on tabular data.}
We preprocess the data of the Adult \cite{Dua2019}, Compass \cite{Angwin2016}, and Default \cite{yeh2009comparisons} datasets and select ``sex'' as the binary sensible attribute denoted with $a$ resulting in $D_\text{fair} := \{\left(x_n, a_n, y_n \right)\}^N_{n=1}$ with $a_n \in \{0, 1\}$. As a differentiable fairness objective we use the hyperbolic tangent relaxation of Difference of Equality of Opportunity ($\widehat{\text{DEO}}$) \cite{padh2020addressing} defined as:

\begin{align}
    \widehat{\text{DEO}} = \frac{1}{N} \sum_{\substack{a=0\\ y=1}} t\left(f(\cdot);c\right) - \frac{1}{N} \sum_{\substack{a=1\\ y=1}} t\left(f(\cdot);c\right)
\end{align}

where $t(x; c)$ denotes $\tanh \left( c \cdot \max \left (0, x \right) \right)$ and set $c=1$ for all experiments.

In compliance with the setup of prior work~\cite{navon2020learning}, we train a Multi-Layer Perceptron with two hidden layers (60 and 25 dimensions) and ReLU activation for 50 epochs. The binary cross-entropy and $\widehat{\text{DEO}}$ represent the multi-objective losses. Furthermore, we use 70\% training, 10\% validation, 20\% test splits, and set $\lambda=0.01$ and $\alpha_1 = \alpha_2 = 0.5$. Due to the high difference in scales in Default dataset we set $\alpha_1 = 0.1$, i.e. sampling more in the vertex of the cross-entropy loss.

\begin{figure*}[htb!]
    \centering
    \includegraphics[width=0.95\textwidth]{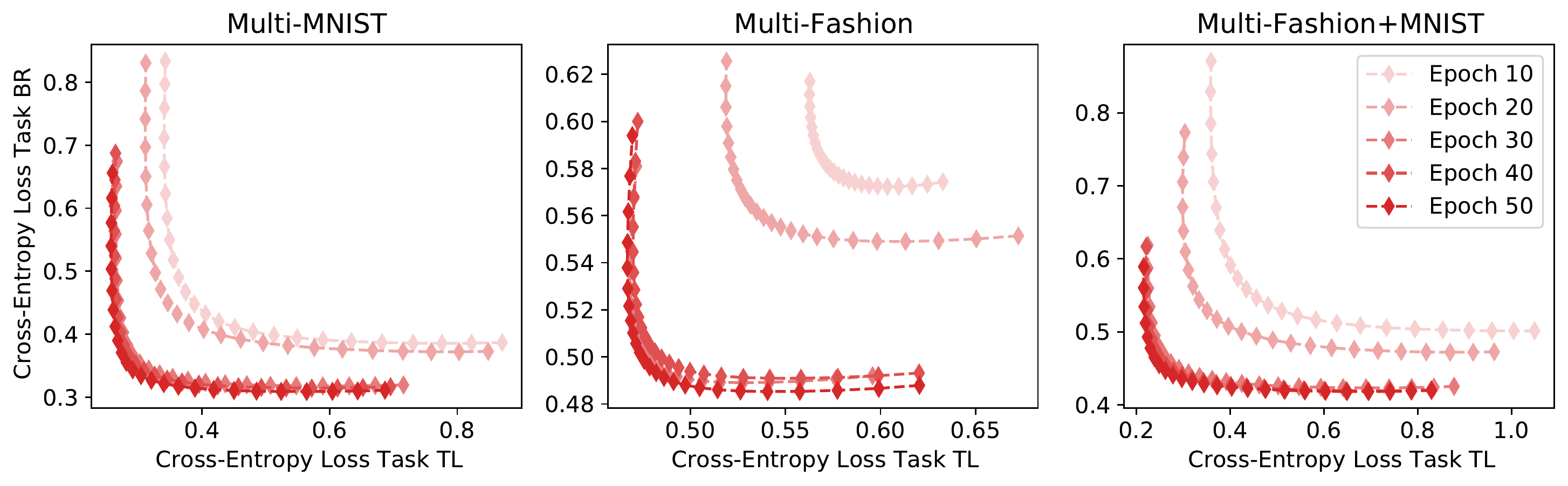}
    \vspace{-0.3cm}
    \caption{Convergence of the Pareto fronts generated by COSMOS on the image classification datasets}
    \label{fig:convergence}
\end{figure*}

\begin{figure*}[htb!]
    \centering
    \includegraphics[width=0.98\textwidth]{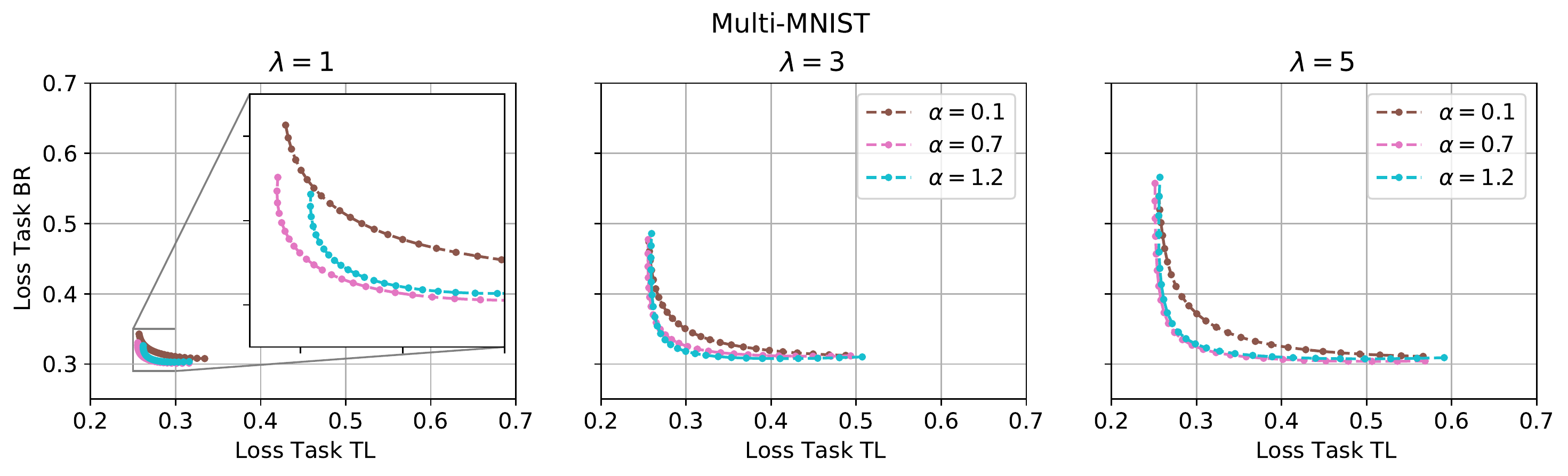}
    \vspace{-0.3cm}
    \caption{Ablation of the impact of the cosine similarity penalty and the Dirichlet sampling parameter on Multi-MNIST}
    \label{fig:ablation}
\end{figure*}

\textbf{Multi-MNIST.} The Multi-MNIST datasets are constructed by overlaying two digits with a slight offset to the bottom right (BR) and top left (TL). The MTL problem is predicting the correct class for both instances at the same time. For a detailed description of the Multi-MNIST dataset and its variants see \citealp{sabour2017dynamic} and \citealp{lin2019pareto}. We use LeNet \cite{lecun1999object} with task-specific heads similar to prior work \cite{sener2018multi} and decay the learning rate at epochs 20, 40, 80, 90 by $0.1$. As losses we define cross-entropy for the BR and TL tasks. We use 10\% of the training data as validation split, and set $\lambda=8$ and $\alpha_1 = \alpha_2 = 1.2$, apart from Multi-Fashion where $\lambda = 2$. For fusing the image $x$ and the preference vector $r$ we transform $r$ to the image space using transposed convolutions and then concatenate the latent vector representation with the image. Concretely, we feed $r$ to a transposed CNN with 2 layers with $J$ hidden dimensions, kernel size $4$ and $6$, and ReLU activation. This results in feature maps of size $J \times 10 \times 10$ which we then upsample to $J \times 36 \times 36$ and append as channels to the image. Because $J=2$, the fusion of $x$ and $r$ yields a $3 \times 36 \times 36$ augmented input to LeNet. 

\textbf{CelebA.} We rescale the images of CelebA~\cite{liu2015faceattributes} to $64 \times 64$ and follow the predefined splits. The deployed prediction model is an ImageNet-pretrained EfficientNet-B4 network~\cite{pmlr-v97-tan19a} (18m parameters) with a single-neuron output layer per objective. We set the learning rate to $5\times10^{-4}$ and batch size to 32. Overall, the fusion of the preference vector with the image creates a $5 \times 64 \times 64$ input tensor. As objectives we define binary cross-entropy for two easy and two hard tasks, utilizing insights of \citealp{sener2018multi}. As hard tasks we pick ``Oval Face'' and ``Pointy Nose'' (A25 and A27), as easy tasks ``Goatee'' and ``Mustache'' (A16 and A22). The hyperparameters for COSMOS are $\lambda = 3$ and $\alpha = 1$; the missclassification rate is computed by taking the values from the center ray (.5, .5) although more elaborated methods are available \cite{wang2017application}. For single task we average the individual scores, and we perform early stopping based on the validation set.

\section{Results}
\label{results}

\textbf{Fairness on tabular data and Multi-MNIST.}

The results of Table~\ref{tab:results_fair_mnist} show that COSMOS compares favorably to the baselines in terms of the quality of the Pareto front measured through the hyper-volume metric. We draw attention that COSMOS has only $5\%$ more parameters than the single-task network, whereas PHN has ca. $100$ times more parameters than the single-task network (see the appendice for details). However, due to the tiny size of the neural networks (2-layer MLP and LeNet) and the small datasets, this huge scalability gap is not reflected proportionally to the training time (a known inefficiency of GPU-based training for tiny models). Nevertheless, the empirical results indicate that COSMOS does not compromise the quality of the Pareto fronts, despite its scalability advantage, addressing \textit{Hypothesis 1}.

Furthermore, COSMOS is not outperformed by the Pareto front baselines (PHN-EPO, PHN-LS and ParetoMTL) on any of the 6 datasets in terms of hyper-volume. 
We ommit hyper-volume for the single task baseline on the fairness datasets because unlike MTL problems, the single-task optima are mutually exclusive. Additionally, there exist trivial solutions for fairness which can achieve perfect scores, e.g. a constant classifier.

For Multi-MNIST beeing a multi-task problem, we indeed outperform the single task baseline in all three datasets, which indicates that our method fulfills \textit{Hypothesis 2}. A further analysis of the Pareto fronts compared to the SOTA methods is shown in Figure~\ref{fig:vssota}, while a comparison to the single task method in Figure~\ref{fig:vssingle}. Moreover, Figure~\ref{fig:convergence} demonstrates that our method converges fast and is able to generate a well-spread Pareto front after 10 epochs.

\textbf{Ablation study.} To analyze the impact of the cosine similarity penalty, we ablate the effect of $\lambda$ and $\alpha$ on the Pareto by varying $\alpha \in (0.1, 0.7, 1.2)$ and $\lambda \in (1, 3, 5)$. Figure~\ref{fig:ablation} clearly demonstrates the crucial importance of the cosine penalty and answers \textit{Hypothesis 3}. In cases where $\lambda$ is small, the Pareto front is at the optimum although it is very narrow. Increasing the penalty yields a wider Pareto front. In contrast, the choice of $\alpha$ influences the quality of the Pareto front (not the width) via the sampling mechanism. For small $\alpha < 1$, we sample more at the vertices of the simplex yielding high-interest points dedicated to each loss. We provide ablations on the other datasets in the appendix. 

\textbf{CelebA.} We show the Pareto fronts for the easy and hard tasks in Figure~\ref{fig:celeba}. The EfficientNet model pretrained on Imagenet transfers very quickly on the Celeb-A dataset. COSMOS is capable of generating a well-spread Pareto front within a few optimization epochs and converges entirely after 25 epochs.

\begin{table}[htb!]
\caption{Results on CelebA}
\begin{center}
\begin{tabular}{l ccc }
\toprule
               & HV   & Mean MCR  & Time (h)\\ \midrule
& \multicolumn{3}{c}{\bf Hard Tasks}  \\ \cmidrule{2-4}
Single Task    & 2.222   & 24.66\%  & 3.08 \\ 
COSMOS         & 2.221   & 25.29\%  & 3.30   \\ 
& \multicolumn{3}{c}{\bf Easy Tasks} \\ \cmidrule{2-4}
Single Task    & 3.719   & 3.15\%  & 3.11 \\ 
COSMOS         & 3.706   & 3.47\%  & 3.30  \\ 

\bottomrule
\end{tabular}
\end{center}
\label{tab:celeba}
\end{table}

In addition, Table~\ref{tab:celeba} presents the comparative analysis against the single-task baseline. COSMOS is qualitatively comparable to the baseline, but in contrast generates a full Pareto front instead of a single point. Moreover, COSMOS computes the front within a small overhead of $7\%$ additional training time when both methods run for 25 epochs.  This confirms \textit{Hypothesis~4} and demonstrates that COSMOS offers a \textit{free lunch} Pareto front approximation, basically at the same time budget required to train a single network on a single task.

\begin{figure}[ht]
    \centering
    \includegraphics[width=0.98\columnwidth]{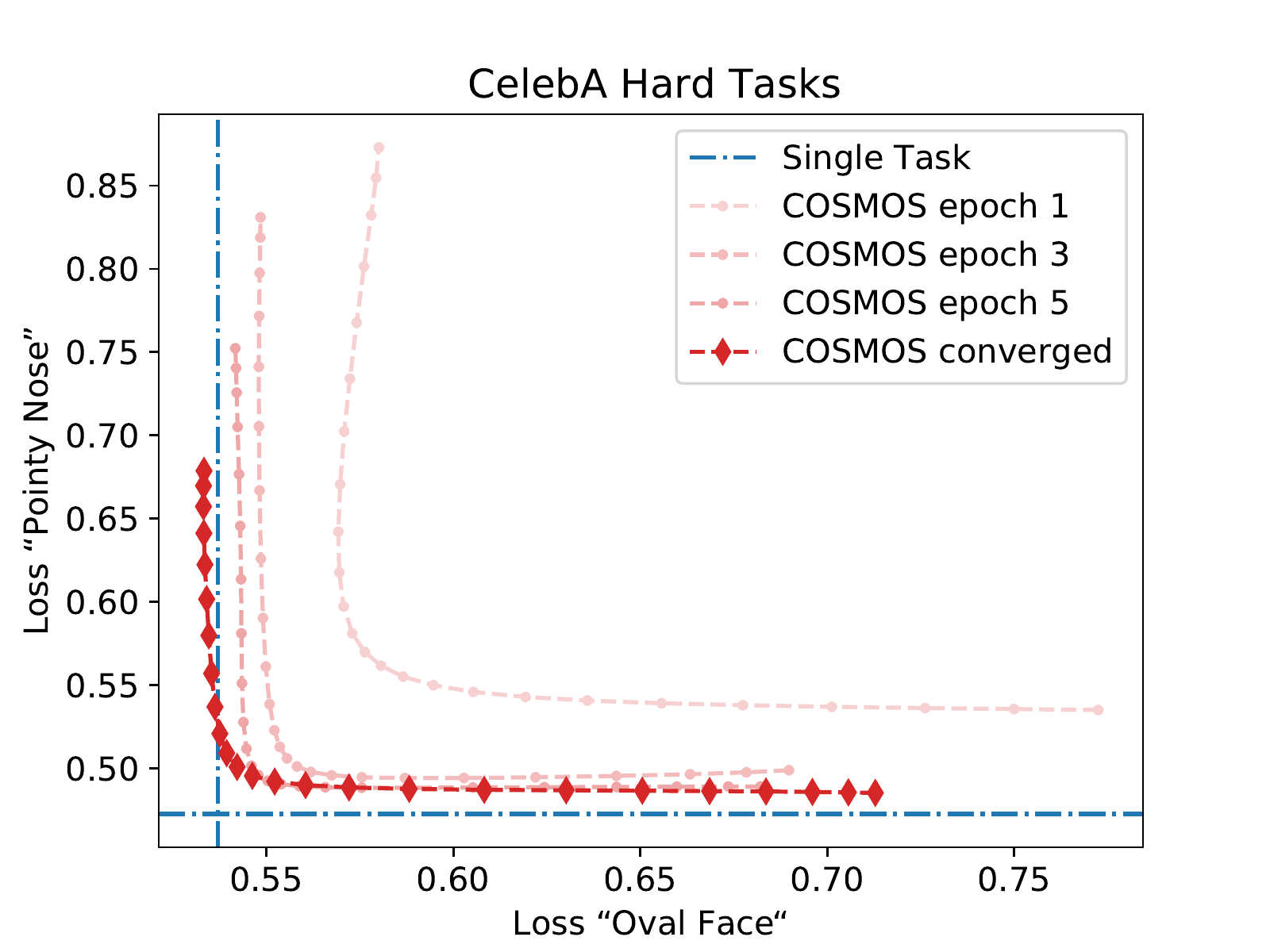}
    \includegraphics[width=0.98\columnwidth]{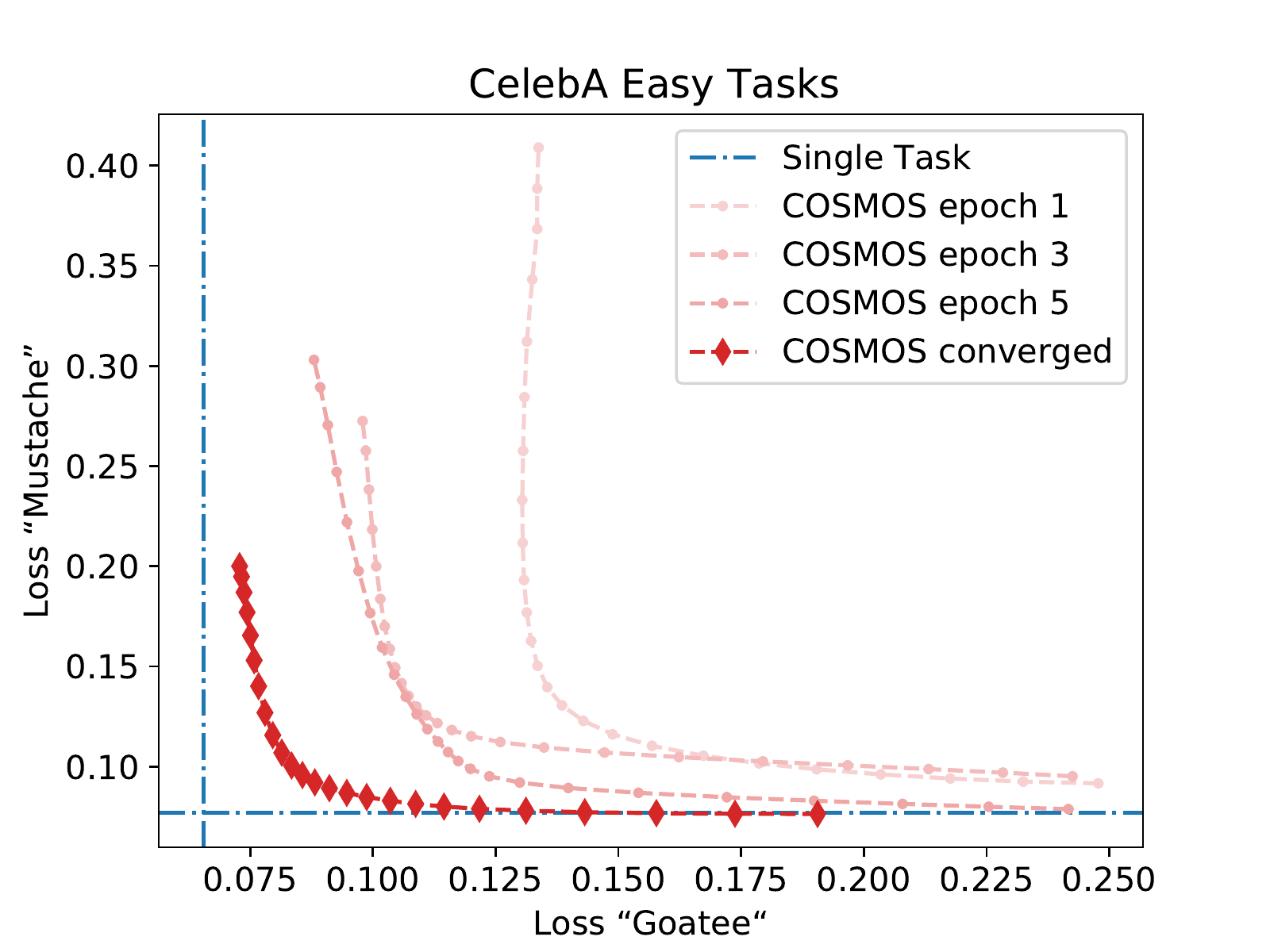}
    \vspace{-0.3cm}
    \caption{Pareto fronts and convergence on CelebA}
    \label{fig:celeba}
\end{figure}

\section{Limitations of our method}

We stress that the Pareto fronts by COSMOS are an approximation, because neural networks are not able to find the exact optimum of the linear scalarization problem. However, the provided illustrations of the approximated Pareto fronts indicate the Pareto-optimality of the achieved solutions for different preference vectors. Secondly, we highlight that conditioning the prediction network to the preference vectors cannot model objectives that do not depend on the input, e.g. an $L_p$ norm regularization objective on weights.

\section{Conclusions}
\label{conclusions}

Multi-objective Optimization is a crucial task for the Machine Learning community due to the wide array of real-life tasks which are intrinsically multi-objective. Unfortunately, there exist no scalable MOO method for deep neural networks which is able to generate a Pareto front of the solutions, in a way that a practitioner can decide on the trade-offs of the competing objectives. Existing papers either relied on learning one network per solution, or deploying hyper-networks which introduce a high number of trainable parameters. In contrast, throughout this paper we introduced COSMOS, a method that is able to produce qualitative approximations of the Pareto front at a very small training time overhead compared to the single-objective optimization of one neural network. We condition the prediction model to the choice of the preference vector for the objectives and solve a regularized version of the linear scalarization problem, in a manner that a single network learns to adjust the predictions based on the inputted trade-off preference for all the objectives. Furthermore, we penalize the Pareto fronts to be widely spread by minimizing the angle between achieved solutions in the space of objectives and the inputted preference vectors. Overall, a single optimization run is required to optimize our multi-objective approach, which is asymptotically equal to the optimization of a single-objective task. In a series of experiments, we demonstrated that the approximated Pareto fronts are competitive against state-of-the-art baselines. Furthermore, we showed that COSMOS approximates a Pareto front for the CelebA dataset using the EfficientNet model with only 7\% of training time overhead compared to single objective optimization tasks. 


\vfill
\newpage
\bibliography{references}
\bibliographystyle{icml2021}

\newpage
\onecolumn

\include{z_appendix}

\end{document}

%% file: math_commands.tex
\usepackage{amsmath,amsfonts,bm}

\def\eqref#1{equation~\ref{#1}}

\def\1{\bm{1}}

\DeclareMathAlphabet{\mathsfit}{\encodingdefault}{\sfdefault}{m}{sl}
\SetMathAlphabet{\mathsfit}{bold}{\encodingdefault}{\sfdefault}{bx}{n}

\newcommand{\Ls}{\mathcal{L}}
\newcommand{\R}{\mathbb{R}}

\DeclareMathOperator*{\argmin}{arg\,min}

%% file: z_appendix.tex






\appendix

\section{More objectives on CelebA}

\paragraph{Three to four objectives.}
We also tried more than two objectives on CelebA using EfficientNet-B4. Here, we report the individual task errors in Table~\ref{tab:celeba_3obj}. As in the main paper we use the ``middle ray'' $\frac{1}{J} \forall j \in \{1, 2, \hdots J\}$ to calculate the Missclassification Rates (MRC) for each task due to its simplicity. Please note that the MCRs do not represent the whole Pareto front but rather one particular point of it. As tasks we use ``Goatee'', ``Mustache'', ``No Beard'', ``Pale Skin'' in this order. We do not report the hypervolume as obtaining evenly distributed points on a ($J-1$)-sphere is not trivial for $J>3$. We again achieve similar performance compared to the single task baseline using the same setting as in the main paper.

The Pareto front for the first three objectives is shown in Figure~\ref{fig:celeba_3obj}. It was obtained using 25 test rays evenly distributed on a 2-sphere using the Fibonacci sphere algorithm.

\paragraph{Ten random objectives.}
To demonstrate that we did not cherry-pick the tasks we also evaluate COSMOS on 10 randomly picked tasks and report the MCR compared to single task. The tasks are: ``Black Hair'', ``Wearing Lipstick'', ``Bald'', ``Goatee'', ``Big Nose'', ``Smiling'', ``Receding Hairline'', ``Sideburns'', ``No Beard'', ``Chubby''. We report the results in Table~\ref{tab:celeba_10obj}. We use the same setting as in the main paper but increased the training epochs to 30.

\begin{figure*}[htb!]
    \centering
    \includegraphics[width=0.48\textwidth]{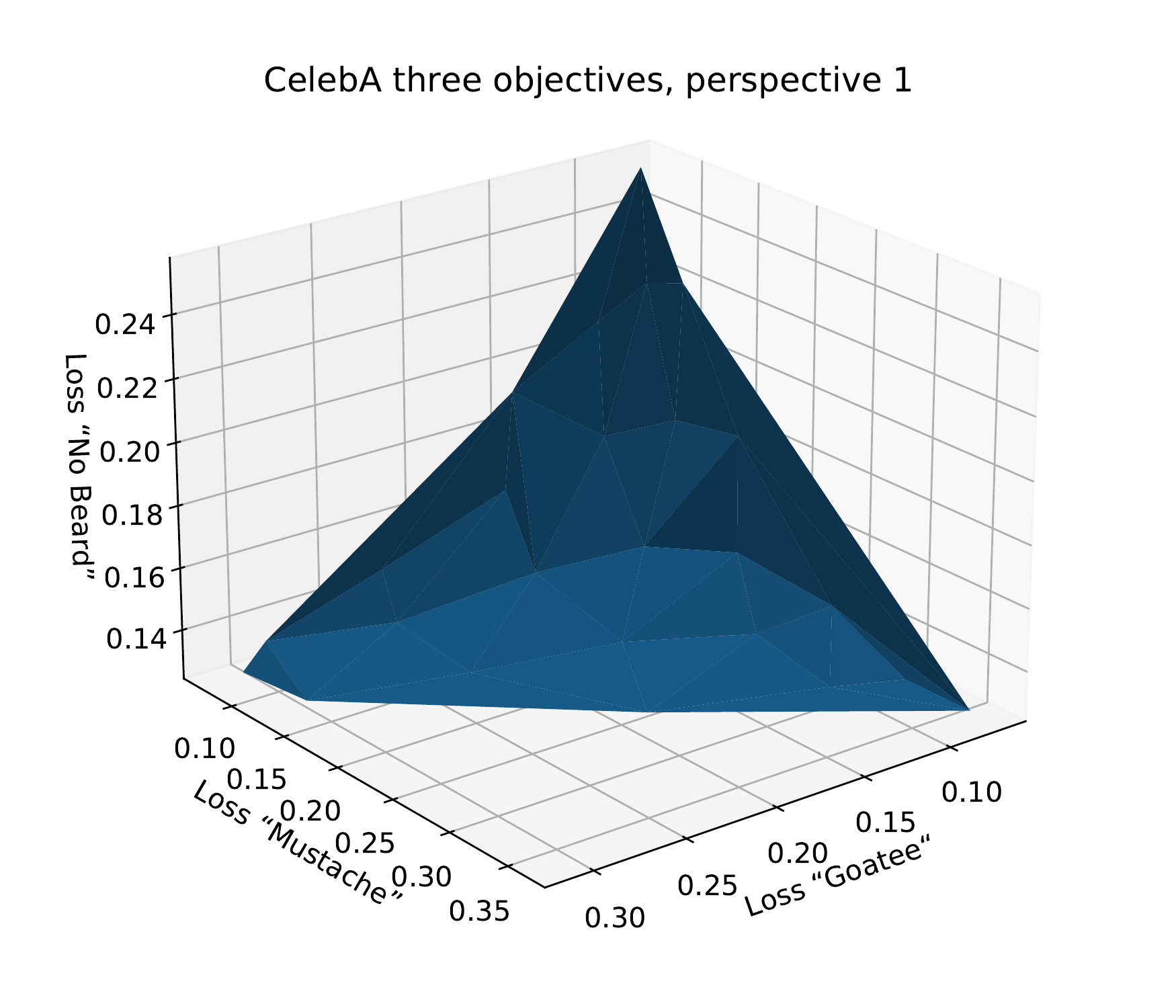}
    \includegraphics[width=0.48\textwidth]{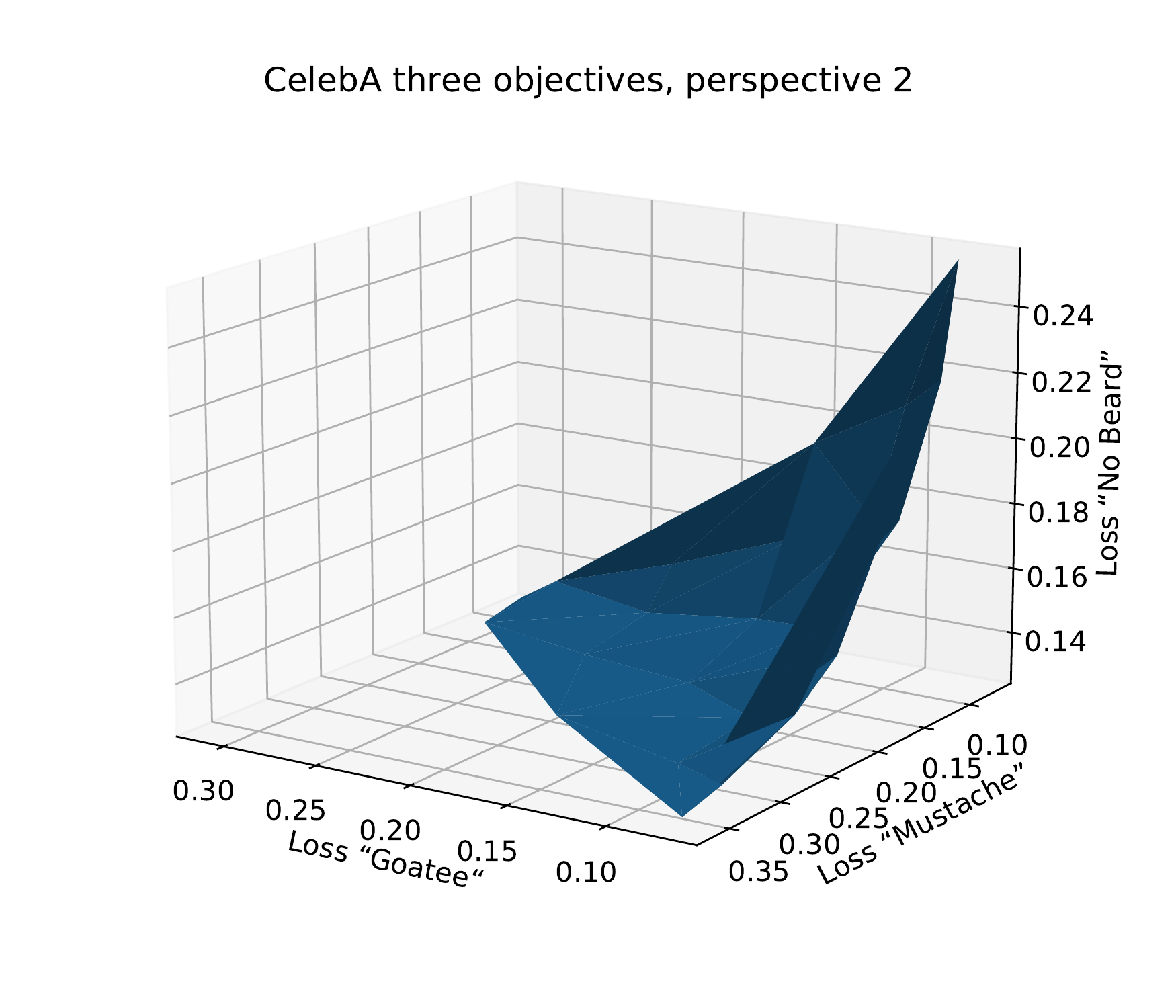}
    \caption{Pareto front found by COSMOS on CelebA with three tasks}
    \label{fig:celeba_3obj}
\end{figure*}

\begin{table}[htb!]
\caption{Results on CelebA with up to 4 objectives}
\begin{center}
\begin{tabular}{l cccc }
\toprule
               & MCR Task 1   & MCR Task 2  & MCR Task 3 &  MCR Task 4   \\ \midrule
2 Tasks        & 3.07\%       & 3.88\%      & -          & - \\ 
3 Tasks        & 3.16\%       & 3.93\%      & 5.28\%     & - \\ 
4 Tasks        & 3.53\%       & 3.59\%      & 5.40\%     & 3.14\% \\
Single Task    & 2.87\%       & 3.30\%      & 4.82\%     & 3.00\% \\

\bottomrule
\end{tabular}
\end{center}
\label{tab:celeba_3obj}
\end{table}

\begin{table}[htb!]
\caption{Results on CelebA with 10 objectives}
\begin{center}
\begin{tabular}{l cc }
\toprule
             & Single Task   & Cosmos     \\ \midrule
Task 1 MCR   & 10.90\%       & 11.63\%    \\ 
Task 2 MCR   & 6.55\%        & 8.87\%     \\ 
Task 3 MCR   & 1.19\%        & 2.00\%     \\ 
Task 4 MCR   & 2.87\%        & 3.24\%     \\ 
Task 5 MCR   & 18.18\%       & 18.06\%    \\ 
Task 6 MCR   & 7.52\%        & 8.84\%     \\ 
Task 7 MCR   & 6.72\%        & 6.87\%     \\ 
Task 8 MCR   & 2.56\%        & 3.43\%     \\ 
Task 9 MCR   & 4.82\%        & 6.05\%     \\ 
Task 10 MCR  & 4.59\%        & 5.26\%     \\ 

\bottomrule
\end{tabular}
\end{center}
\label{tab:celeba_10obj}
\end{table}

\section{Intuition behind the Cosine Similarity penalty}

We would also like to add two illustrations that demonstrate the intuition behind deploying the cosine similarity as a penalty. Figure~\ref{fig:scalarization_issue} shows the difference between the Pareto front found using regular linear scalarization and what is actually desired.

Figure~\ref{fig:angel_min} shows how the penalty ensuring a small angle between the loss vector $\ell$ and the preference vector $\vec r$ pushes the solution closer towards the desired position while still minimizing both losses. This way minimizing the losses as well as reaching a position on the Pareto front are balanced during optimization.

\begin{figure*}[ht!]
    \centering
    \includegraphics[trim=0 0 0 60, clip, width=1.0\textwidth]{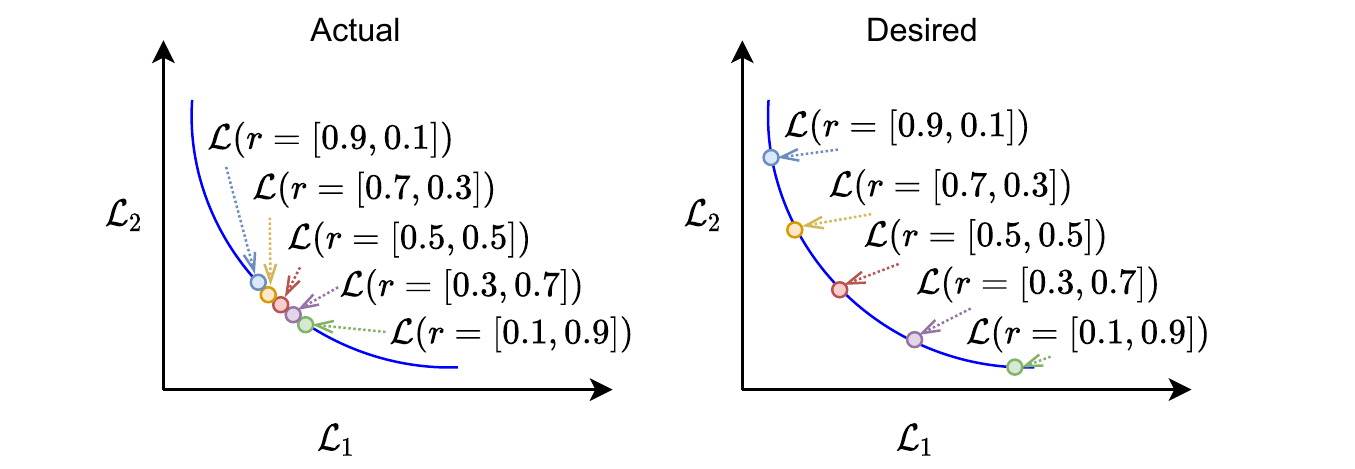}
    \vspace{-0.3cm}
    \caption{Pareto front found by linear scalarization (left) vs. desired (right)}
    \label{fig:scalarization_issue}
\end{figure*}

\begin{figure*}[ht!]
    \centering
    \includegraphics[trim=0 0 0 60, clip, width=1.0\textwidth]{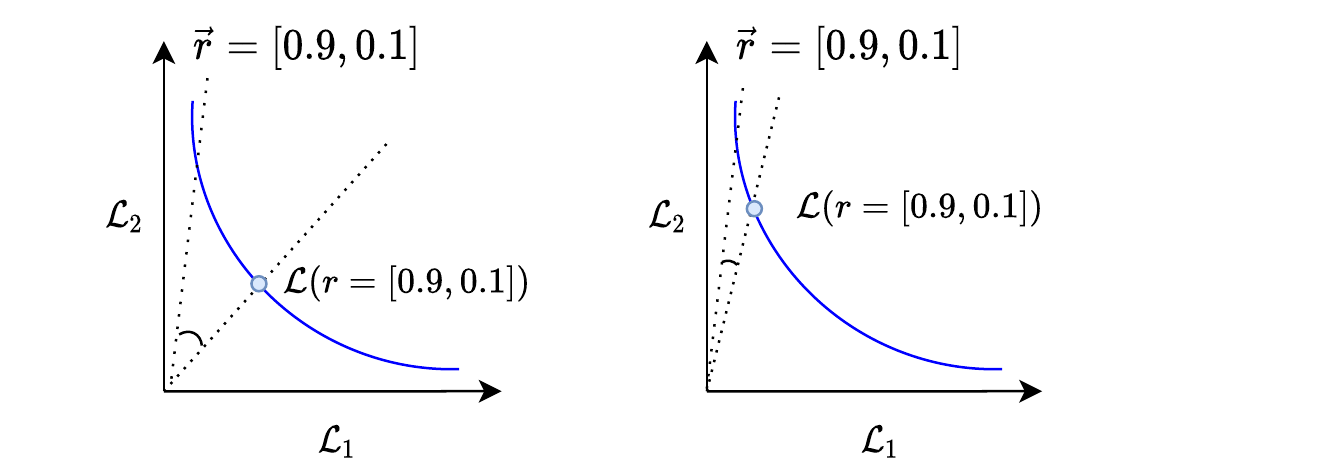}
    \vspace{-0.3cm}
    \caption{The effect of angle minimization on achieving a wide spread Pareto front. }
    \label{fig:angel_min}
\end{figure*}

\section{Additional figures}

See Figure~\ref{fig:vssingle} for a comparison of COSMOS to the single task baseline for the fairness  datasets.

See Figures~\ref{fig:ablation} for an ablation of $\lambda$ and $\alpha$ on the datasets ommitted in the main paper. For the fairness datasets we ablate $\lambda \in \{0, 0.01, 0.1 \}$. They demonstrate the same behaviour as for Multi-MNIST, shown in the paper, although the cosine similarity seems not as important for fairness as for the MTL tasks. This is plausible due to the large difference in scales.

\begin{figure*}[ht!]
    \centering
    \includegraphics[width=0.95\textwidth]{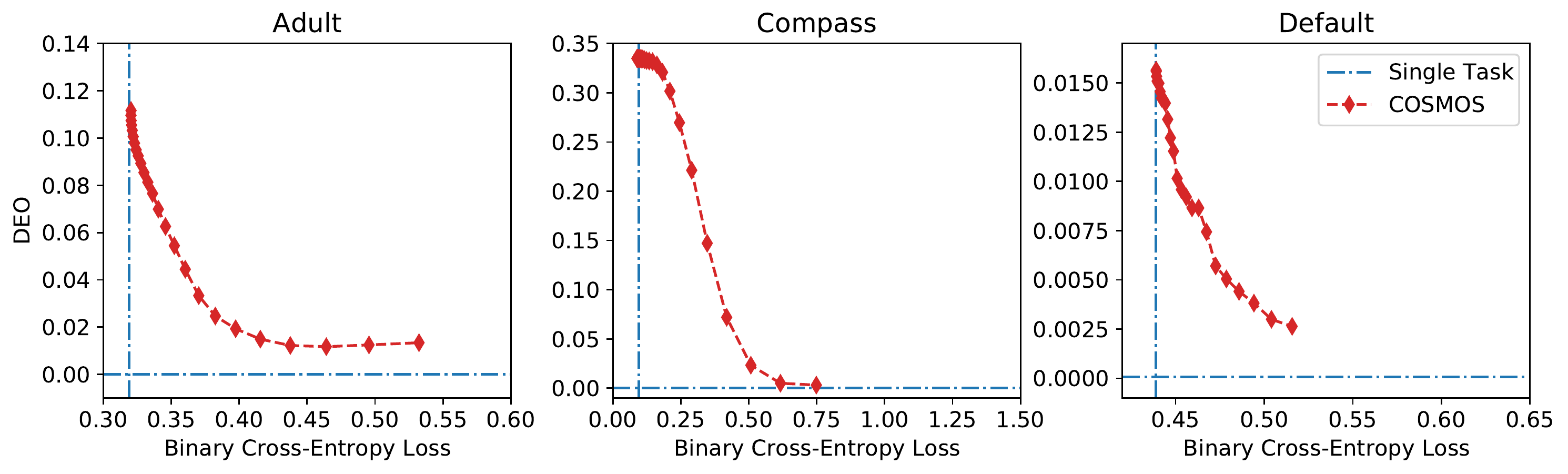}
    \vspace{-0.3cm}
    \caption{Comparison of COSMOS against single-objective baselines}
    \label{fig:vssingle}
\end{figure*}

\begin{figure*}[ht!]
    \centering
    \includegraphics[width=0.93\textwidth]{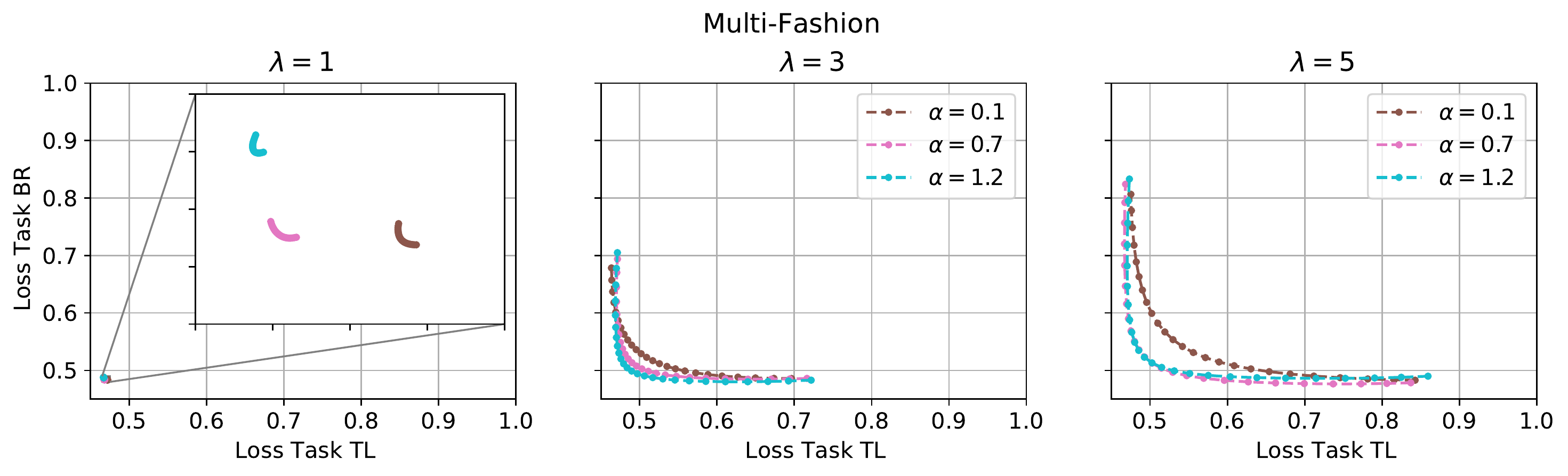}
    \includegraphics[width=0.93\textwidth]{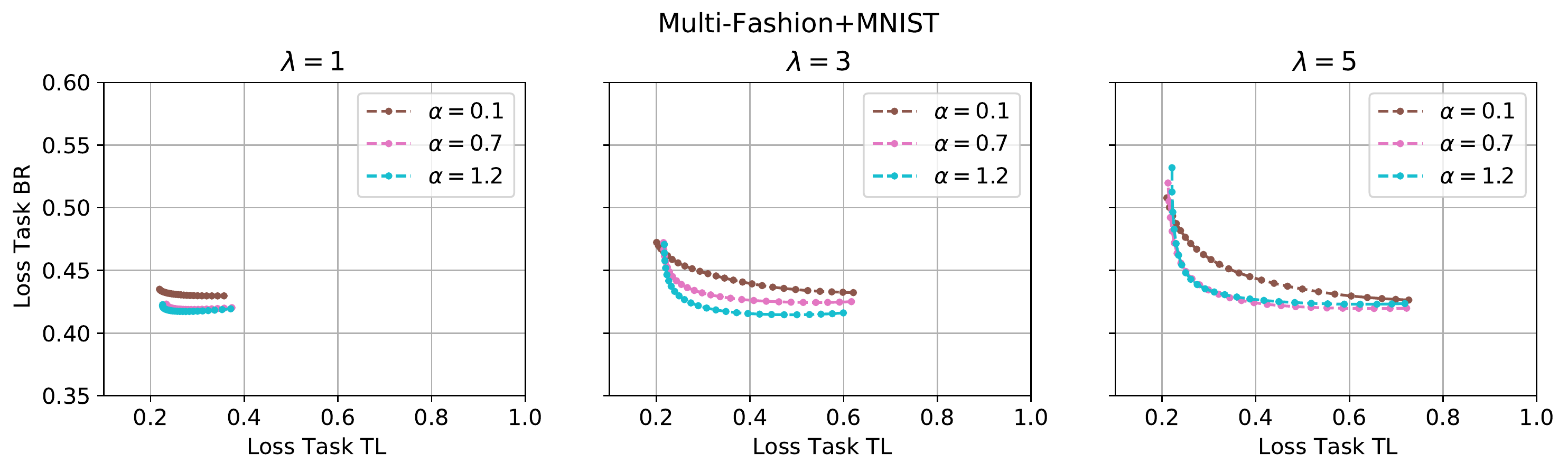}
    \includegraphics[width=0.93\textwidth]{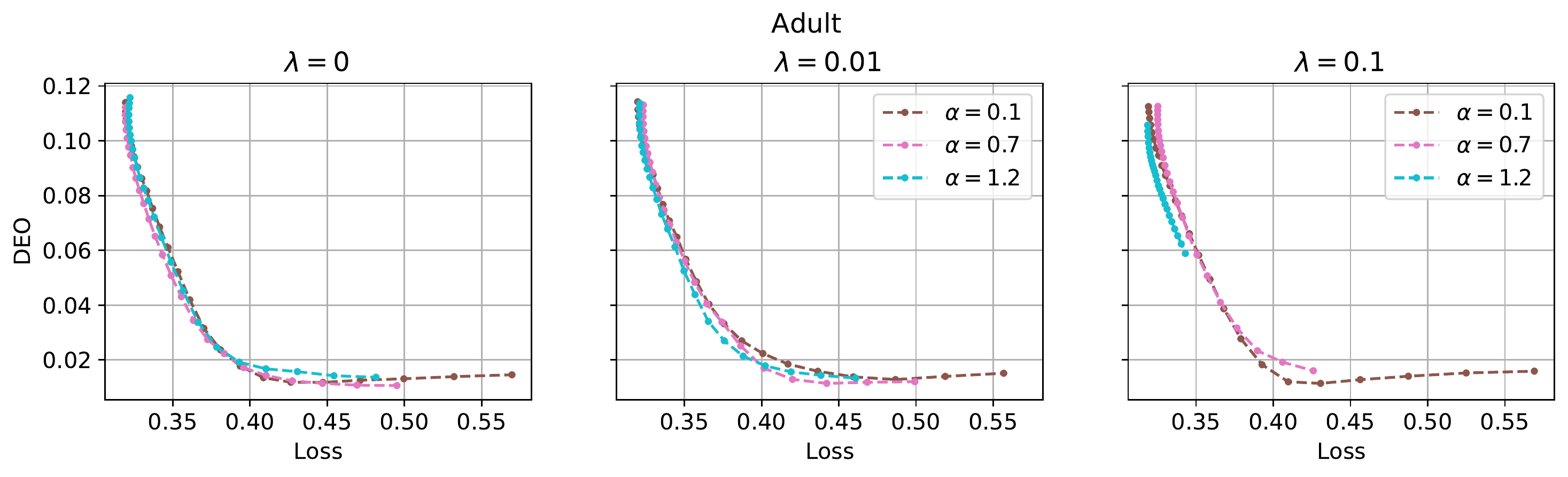}
    \includegraphics[width=0.93\textwidth]{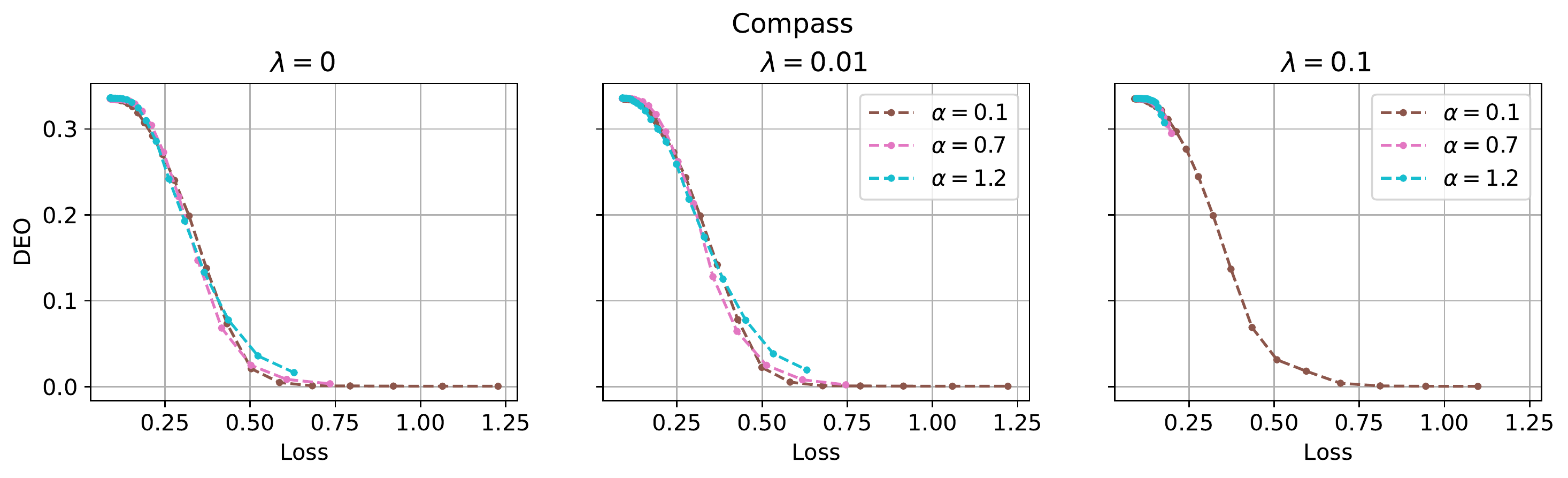}
    \includegraphics[width=0.93\textwidth]{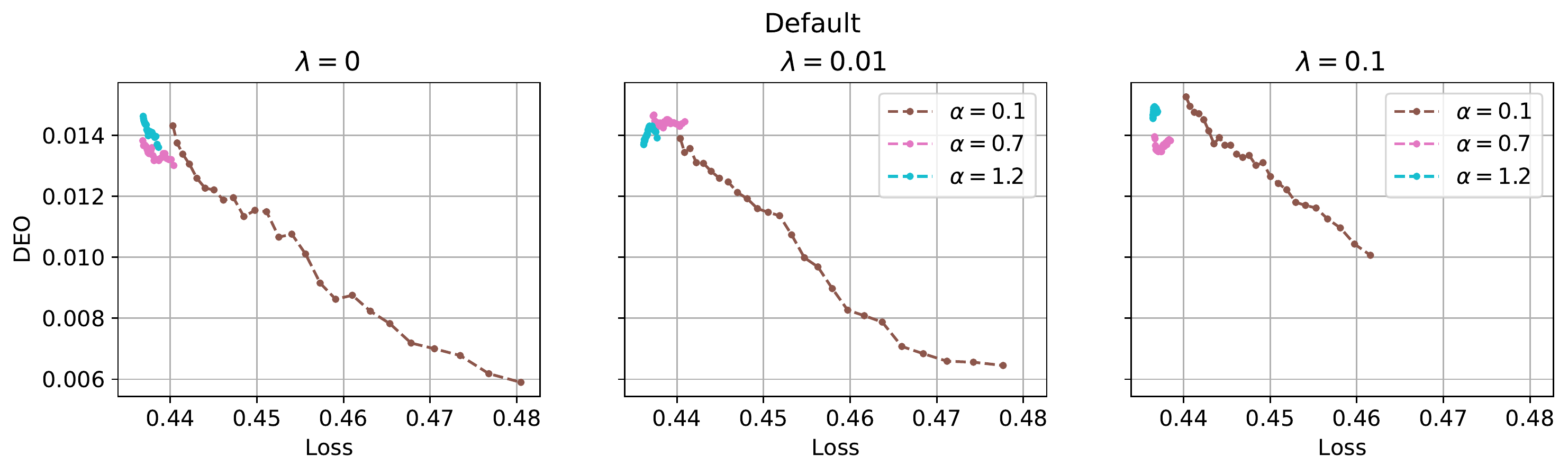}
    \vspace{-0.3cm}
    \caption{Ablation of the impact of the cosine similarity penalty and the Dirichlet sampling parameter.}
    \label{fig:ablation}
\end{figure*}

\section{More detailed result tables}

For more detailed results on the three fairness and image classification datasets see Table~\ref{tab:fairness}.

\begin{table}[htb!]
\begin{center}
\caption{Detailed results on fairness and image classification datasets}
\label{tab:fairness}
\begin{tabular}{l cr r}
\toprule
                & Hyper Vol. & Time (Sec) & \# Params. \\ \midrule
                & \multicolumn{3}{c}{\bf Adult} \\ \cmidrule{2-4}
Single Task    & 3.36 $\pm$ 0.01        & 48          &  6k \\ 
PHN-EPO    & 3.34 $\pm$ 0.00        & 59          &  716k \\ 
PHN-LS    & 3.34 $\pm$ 0.00        & 23          &  716k \\ 
ParetoMTL    & 2.90 $\pm$ 0.05        & 539          &  6k \\ 
COSMOS    & 3.34 $\pm$ 0.01        & 31          &  7k \\ 
                & \multicolumn{3}{c}{\bf Compass} \\ \cmidrule{2-4}
Single Task    & 3.81 $\pm$ 0.00        & 16          &  2k \\ 
PHN-EPO    & 3.71 $\pm$ 0.01        & 17          &  304k \\ 
PHN-LS    & 3.71 $\pm$ 0.01        & 7          &  304k \\ 
ParetoMTL    & 2.15 $\pm$ 0.37        & 83          &  2k \\ 
COSMOS    & 3.72 $\pm$ 0.01        & 17          &  2k \\ 
                & \multicolumn{3}{c}{\bf Default} \\ \cmidrule{2-4}
Single Task    & 3.12 $\pm$ 0.00        & 34          &  7k \\ 
PHN-EPO    & 3.11 $\pm$ 0.01        & 27          &  728k \\ 
PHN-LS    & 3.12 $\pm$ 0.00        & 18          &  728k \\ 
ParetoMTL    & 3.10 $\pm$ 0.00        & 334          &  7k \\ 
COSMOS    & 3.12 $\pm$ 0.00        & 35          &  7k \\ 
                & \multicolumn{3}{c}{\bf Multi-MNIST} \\ \cmidrule{2-4}
Single Task    & 2.85 $\pm$ 0.01        & 391          &  42k \\ 
PHN-EPO    & 2.83 $\pm$ 0.03        & 1,554          &  3,243k \\ 
PHN-LS    & 2.83 $\pm$ 0.03        & 636          &  3,243k \\ 
ParetoMTL    & 2.90 $\pm$ 0.01        & 4,087          &  42k \\ 
COSMOS    & 2.94 $\pm$ 0.02        & 501          &  43k \\ 
                & \multicolumn{3}{c}{\bf Multi-Fashion} \\ \cmidrule{2-4}
Single Task    & 2.23 $\pm$ 0.01        & 421          &  42k \\ 
PHN-EPO    & 2.20 $\pm$ 0.03        & 1,852          &  3,243k \\ 
PHN-LS    & 2.20 $\pm$ 0.03        & 700          &  3,243k \\ 
ParetoMTL    & 2.24 $\pm$ 0.02        & 4,210          &  42k \\ 
COSMOS    & 2.32 $\pm$ 0.01        & 379          &  43k \\ 
                & \multicolumn{3}{c}{\bf Multi-Fashion+MNIST} \\ \cmidrule{2-4}
Single Task    & 2.77 $\pm$ 0.02        & 398          &  42k \\ 
PHN-EPO    & 2.78 $\pm$ 0.02        & 1,715          &  3,243k \\ 
PHN-LS    & 2.76 $\pm$ 0.05        & 723          &  3,243k \\ 
ParetoMTL    & 2.72 $\pm$ 0.02        & 4,142          &  42k \\ 
COSMOS    & 2.83 $\pm$ 0.03        & 498          &  43k \\ 
\bottomrule
\end{tabular}
\end{center}
\end{table}

\section{Experimental Setup (further details)}

We ran all experiments on a Nvidia GeForce RTX 2080 GPU with 12GB RAM. As seeds we used $1, 2, 3, 4, 42$.

\paragraph{Fairness on tabular data.}
After preprocessing the datasets have the following dimensions: Adult $48,842 \times 88$, Compass $6,172 \times 20$, Default $30,000 \times 90$. For HPN where we use learning rate 0.0001 set as default by the authors.

\paragraph{Multi-MNIST.}
For ParetoMTL we decay at 15,30,45,60,75,90 it to $\gamma = 0.5$ following the authors' implementation. We train HPN for 150 epochs with learning rate 0.0001, no scheduler and $\alpha=0.2$, again following their implementation.

\paragraph{CelebA.}
To transform the images we first resize so the smaller edge matches 64 pixels. Then we crop in the center and normalize to 0.5 mean and 0.5 standard deviation. We remove all Batchnorm layers.
